\documentclass[twoside,11pt]{article}

\usepackage{jmlr2e}
\usepackage{times}
\usepackage{graphicx} 
\usepackage{tabularx}
\usepackage{amsmath}
\usepackage{caption}
\usepackage{todonotes}
\usepackage{amssymb}
\newtheorem{assumption}{Assumption}

\usepackage{natbib}

\usepackage{algorithm}
\usepackage{algorithmic}

\firstpageno{1}

\begin{document}

\title{Learning Linear Non-Gaussian Causal Models in the Presence of Latent Variables}

\author{\name Saber Salehkaleybar \email saleh@sharif.edu \\
       \addr Department of Electrical Engineering\\
       Sharif University of Technology, Tehran, Iran
       \AND
       \name AmirEmad Ghassami \email ghassam2@illinois.edu \\
       \addr Department of ECE\\ University of Illinois at Urbana-Champaign\\
       Urbana, IL 61801     
       \AND
       \name Negar Kiyavash \email negar.kiyavash@gatech.edu \\
       \addr Departments of ECE and ISE\\ 
       Georgia Institute of Technology\\
       Atlanta, GA 30308
       \AND
       \name Kun Zhang \email kunz1@cmu.edu \\
       \addr Department of Philosophy\\ Carnegie Mellon University\\
       Pittsburgh, PA 15213
   }

\editor{ }

\maketitle

\begin{abstract}
We consider the problem of learning causal models from observational data generated by linear non-Gaussian acyclic causal models with latent variables. Without considering the effect of latent variables, one usually infers wrong causal relationships among the observed variables. 
Under faithfulness assumption, we propose a method to check whether there exists a causal path between any two observed variables. From this information, we can obtain the causal order among them.
The next question is then whether or not the causal effects can be uniquely identified as well.
It can be shown that causal effects among observed variables cannot be identified uniquely even under the assumptions of faithfulness and non-Gaussianity of exogenous noises. However, we will propose an efficient method to identify the set of all possible causal effects that are compatible with the observational data. Furthermore, we present some structural conditions on the causal graph under which we can learn causal effects among observed variables uniquely. We also provide necessary and sufficient graphical conditions for unique identification of the number of variables in the system.
 Experiments on synthetic data and real-world data show the effectiveness
of our proposed algorithm on learning causal models.
\end{abstract}


\section{Introduction}
One of the primary goals in empirical sciences is to discover casual relationships among a set of variables of interest in various natural and social phenomena. Such causal relationships can be recovered by conducting controlled experiments. However, performing controlled experiments is often expensive or even impossible due to technical or ethical reasons. Thus, it is vital to develop statistical methods for recovering causal relationships from non-experimental data. 


Probabilistic graphical models are commonly used to represent causal relations. Alternatively, Structural Equation Models (SEM) which further  specify mathematical equations among the variables can be used to represent probabilistic causal influences. Linear SEMs are a special class of SEMs where each variable is a linear combination of  its direct causes and an exogenous noise. Under the causal sufficiency assumption, by utilizing conventional causal structure learning algorithms such as PC \citep{spirtes2000causation} and IC \citep{pearl2009causality}, we can identify a class of models that are equivalent in the sense that they represent the same set of conditional independence assertions obtained from data. If we have background knowledge about the data-generating mechanism, we may further narrow down the possible models that are compatible with the observed data \citep{peters2016causal,ghassami2017budgeted,salehkaleybar2017learning,zhang2017causal,peters2013identifiability,zhang2009identifiability,hoyer2009nonlinear,janzing2012information}. For instance, \cite{shimizu2006linear} proposed a linear non-Gaussian acyclic model (LiNGAM) discovery algorithm that can identify causal structure uniquely by assuming non-Gaussian distributions for the exogenous noises in the linear SEM model. However, LiNGAM algorithm and its regression-based variant (DirectLiNGAM) \citep{shimizu2011directlingam} rely on the causal sufficiency assumption, i.e., no unobserved common causes exist for any pair of variables that are under consideration in the model.

In the presence of latent variables, \cite{hoyer2008estimation} showed that linear SEM can be converted to a canonical form where each latent variable has at least two children and no parents. Such latent variables are commonly called ``latent confounders". Furthermore, they proposed a solution which casts the problem of identifying causal effects among observed variables into an overcomplete-ICA problem and returns multiple causal structures that are observationally equivalent. The time complexity of searching such structures can be as high as ${p \choose p_o}$ where $p_o$ and $p$ are the number of observed and total variables in the system, respectively.
\cite{entner2010discovering} proposed a method that identifies a partial causal structure among the observed variables by recovering all the unconfounded sets\footnote{A set of variables is called unconfounded if there is no variable outside the set which is confounder of some variables in the set. In Figure \ref{fig:exampleintro}, variable $V_3$ is a confouder of variables $V_1$ and $V_2$ but it is not observable. Thus, the set of variables $V_1$ and $V_2$ is not unconfounded.} and then learning the causal effects for each pair of variables in the set. However, their method may return an empty unconfounded set if latent confounders are the cause of most of observed variables in the system such as the simple example of Figure \ref{fig:exampleintro}. \cite{chen2013causality} showed that a causal order and causal effects among observed variables can be identified if the latent confounders have Gaussian distribution and exogenous noises of observed variables are simultaneously super-Gaussian or sub-Gaussian.
\begin{figure}
	\centering
	\includegraphics[width=0.2\textwidth]{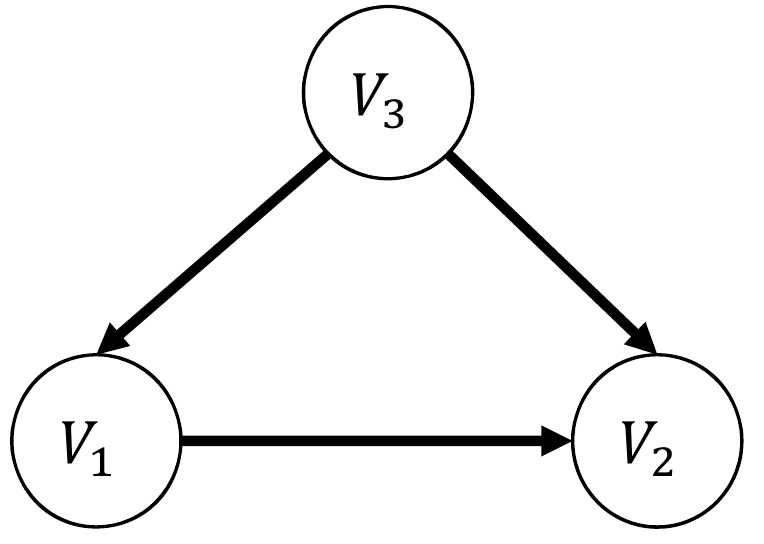}
	\caption{An example of causal graphs: $V_1$ and $V_2$ are observed variables while $V_3$ is latent. }\label{fig:exampleintro}
\end{figure} 
In \citep{tashiro2014parcelingam}, the ideas in DirectLiNGAM was extended to the case where latent confounders exist in the system. The proposed solution first tries to find a root variable (a variable with no parents). Then, the effect of such variable is removed by regressing it out. This procedure continues until any variable and its residual becomes dependent. Subsequently, a similar iterative procedure is used to find a sink variable and remove its effect from other variables. However, this solution may not recover causal order in some causal graphs such as the one in Figure \ref{fig:exampleintro}.\footnote{In Figure \ref{fig:exampleintro}, the root variable ($V_3$) is latent and the regressor of sink variable $V_2$ and the residual are not independent without considering the latent variable $V_3$ in the set of regressors. Thus, no root or sink variable can be identified in the system.}. \cite{shimizu2014bayesian} proposed a Bayesian approach for estimating the causal direction between two observed variables when the sum of non-Gaussian independent latent confounders has a multivariate $t$-distribution. They compute log-marginal likelihoods to infer causal directions. 


Rather surprisingly, although the causal structure is in general not fully identifiable in the presence of latent variables,  we will show that the causal order among the observed  variables is still identifiable under the faithfulness assumption. 
In order to obtain a causal order, we first check whether there exists a causal path between any two observed variables. Subsequently, from this information, we obtain a causal order among them.
Having established a causal order, we aim to figure out whether the causal effects are uniquely identifiable from observational data.
We show by an example that causal effects among observed variables is not uniquely identifiable even if the faithfulness assumption holds true and the exogenous noises are non-Gaussian. We propose a method to identify the set of all possible causal effects efficiently in time that are compatible with the observational data. Furthermore, we present some structural conditions on the causal graph under which causal effects among the observed variables can be identified uniquely. We also provide necessary and sufficient graphical conditions under which the number of latent variables is uniquely identifiable.




The rest of this paper is organized as follows. In Section 2, we define the problem of identifying causal orders and causal effects in linear causal systems with latent variables. In Section 3, we propose our approach to learn the causal order among the observed variables and provide necessary and sufficient graphical conditions under which the number of latent variables is uniquely identifiable. In Section 4, we present a method to find the set of all possible causal effects which are consistent with the observational data and give conditions under which causal effects are uniquely identifiable. We conduct experiments to evaluate
the performance of proposed solutions in
Section 5 and conclude in Section 6.




\section{Problem Definition}

\subsection{Notations}

In a directed graph $G = (\mathcal{V},E)$ with the vertex set $\mathcal{V}=\{V_1,\cdots, V_p\}$ and the edge set $E$, we denote a directed edge from $V_i$ to $V_j$ by $(V_i,V_j)$. A directed path $P=(V_{i_0}, V_{i_1}, \cdots, V_{i_k})$ in $G$ is a sequence of vertices of $G$ where there is a directed edge from $V_{i_j}$ to $V_{i_{j+1}}$ for any $0\leq j \leq k-1$. We define the set of variables $\{V_{i_1},\cdots, V_{i_{k-1}}\}$ as the intermediate variables on the path $P$. We use notation $V_i \rightsquigarrow V_j$ to show that there exists a directed path from $V_i$ to $V_j$. If there is a directed path from $V_i$ to $V_j$, $V_i$ is ancestor of $V_j$ and that $V_j$ is a descendant of $V_i$.  More formally, $anc(V_i)=\{V_j|V_j\rightsquigarrow V_i\}$ and $des(V_i)=\{V_j|V_i\rightsquigarrow V_j\}$. Each variable $V_i$ is an ancestor and a descendant of itself. 

We denote vectors and matrices by boldface letters. The vectors $\mathbf{A}_{i,:}$ and $\mathbf{A}_{:,i}$ represent $i$-th row and column of matrix $\mathbf{A}$, respectively. The $(i,j)$ entry of matrix $\mathbf{A}$ is denoted by $[\mathbf{A}]_{i,j}$. For $n\times m$ matrix $\mathbf{A}$ and $n\times p$ matrix $\mathbf{B}$, the notation $[\mathbf{A},\mathbf{B}]$ denotes the horizontal concatenation. For $n\times m$ matrix $\mathbf{A}$ and $p\times m$ matrix $\mathbf{B}$, the notation $[\mathbf{A};\mathbf{B}]$ shows the vertical concatenation.

\subsection{System Model}

Consider a linear SEM among a set of variables $\mathcal{V}=\{V_1,\cdots,V_{p}\}$:

\begin{equation}
\mathbf{V}=\mathbf{A}\mathbf{V}+\mathbf{N},
\label{eq:model}
\end{equation}
where the vectors $\mathbf{V}$ and $\mathbf{N}$ denote the random variables in $\mathcal{V}$ and their corresponding exogenous noises, respectively. The entry $(i,j)$ of matrix $\mathbf{A}$ shows the strength of direct causal effect of variable $V_j$ on variable $V_i$. We assume that the causal relations among random variables can be represented by a directed acyclic graph (DAG). Thus, the variables in $\mathcal{V}$ can be arranged in a causal order, such that no latter variable causes any earlier variable. We denote such a causal order on the variables by $k$ in which $k(i), i\in\{1,\cdots,p\}$ shows the position of variable $V_i$ in the causal order. The matrix $\mathbf{A}$ can be converted to a strictly lower triangular matrix by permuting its rows and columns simultaneously based on the causal order. 

\begin{example}
	Consider the following linear SEM with four random variables $\{V_1,\cdots$ $,V_4\}$:
	\begin{equation*}
	\left[ {\begin{array}{c}
		V_1 \\
		V_2 \\
		V_3 \\
		V_4 \\
		\end{array} } \right]=
	\left[ {\begin{array}{cccc }
		0 & e & 0 & d\\
		0 & 0 & 0 & 0\\
		0 & a & 0 & 0\\
		0 & b & c & 0\\
		\end{array} } \right]
	\left[ {\begin{array}{c}
		V_1 \\
		V_2 \\
		V_3 \\
		V_4 \\
		\end{array} } \right]+
	\left[ {\begin{array}{c}
		N_{1} \\
		N_{2} \\
		N_{3} \\
		N_{4} \\
		\end{array} } \right],
	\label{eq:ex1}
	\end{equation*}
	where $a,b,c,d$ and $e$ are some constants (see Figure \ref{fig:ex1}). A causal order in this SEM model would be: $k(1)=4,k(2)=1,k(3)=2,k(4)=3$. Hence, the matrix $\mathbf{P}\mathbf{A}\mathbf{P}^T$ is strictly lower triangular where $\mathbf{P}$ is a permutation matrix associated with $k$ defined by the following non-zero entries: $\{(k(i),i)| 1\leq i\leq 4\}$. 
	

	%
	\label{ex:1}
\end{example}


We split random variables in $\mathbf{V}$ into an observed vector $\mathbf{V_o}\in \mathbb{R}^{p_o}$ and a latent vector $\mathbf{V_l}\in \mathbb{R}^{p_l}$ where $p_o$ and $p_l$ are the number of observed and latent variables, respectively.
Without loss of generality, we assume that first $p_o$ entries of $\mathbf{V}$ are observable, i.e. $\mathbf{V_o}=[V_1,\cdots,V_{p_o}]^T$ and $\mathbf{V_l}=[V_{p_o+1},\cdots,V_p]^T$. Therefore,
\begin{equation}
\left[ {\begin{array}{c}
	\mathbf{V_o} \\
	\mathbf{V_l} \\
	\end{array} } \right]=
\left[ {\begin{array}{cc }
	\mathbf{A_{oo}} &\mathbf{A_{ol}} \\
	\mathbf{A_{lo}} &\mathbf{A_{ll}} \\
	\end{array} } \right]\left[ {\begin{array}{c}
	\mathbf{V_o} \\
	\mathbf{V_l} \\
	\end{array} } \right]+
\left[ {\begin{array}{c}
	\mathbf{N_o} \\
	\mathbf{N_l} \\
	\end{array} } \right],
\label{eq:main}
\end{equation}
where $\mathbf{N_o}$ and $\mathbf{N_l}$ are the vectors of exogenous noises of $\mathbf{V_o}$ and $\mathbf{V_l}$, respectively.   
Furthermore, we have: $\mathbf{A}=[\mathbf{A_{oo}},\mathbf{A_{ol}};\mathbf{A_{lo}},\mathbf{A_{ll}}]$.

The causal order among all variables $k$, induces a causal order $k_o$ among the observed variables as follows:
For any two observed variables $V_i,V_j$, $1\leq i,j\leq p_o$, $k_o(i)<k_o(j)$ if $k(i)<k(j)$. Similarly, $k$ induces a causal order among latent variables. We denote this causal order by $k_l$. It can be easily shown that $\mathbf{A_{oo}}$ and $\mathbf{A_{ll}}$ can be converted to strictly lower triangular matrices by permuting rows and columns simultaneously based on causal orders $k_o$ and $k_l$, respectively.

\begin{figure}
	\centering
	\includegraphics[width=0.33\textwidth]{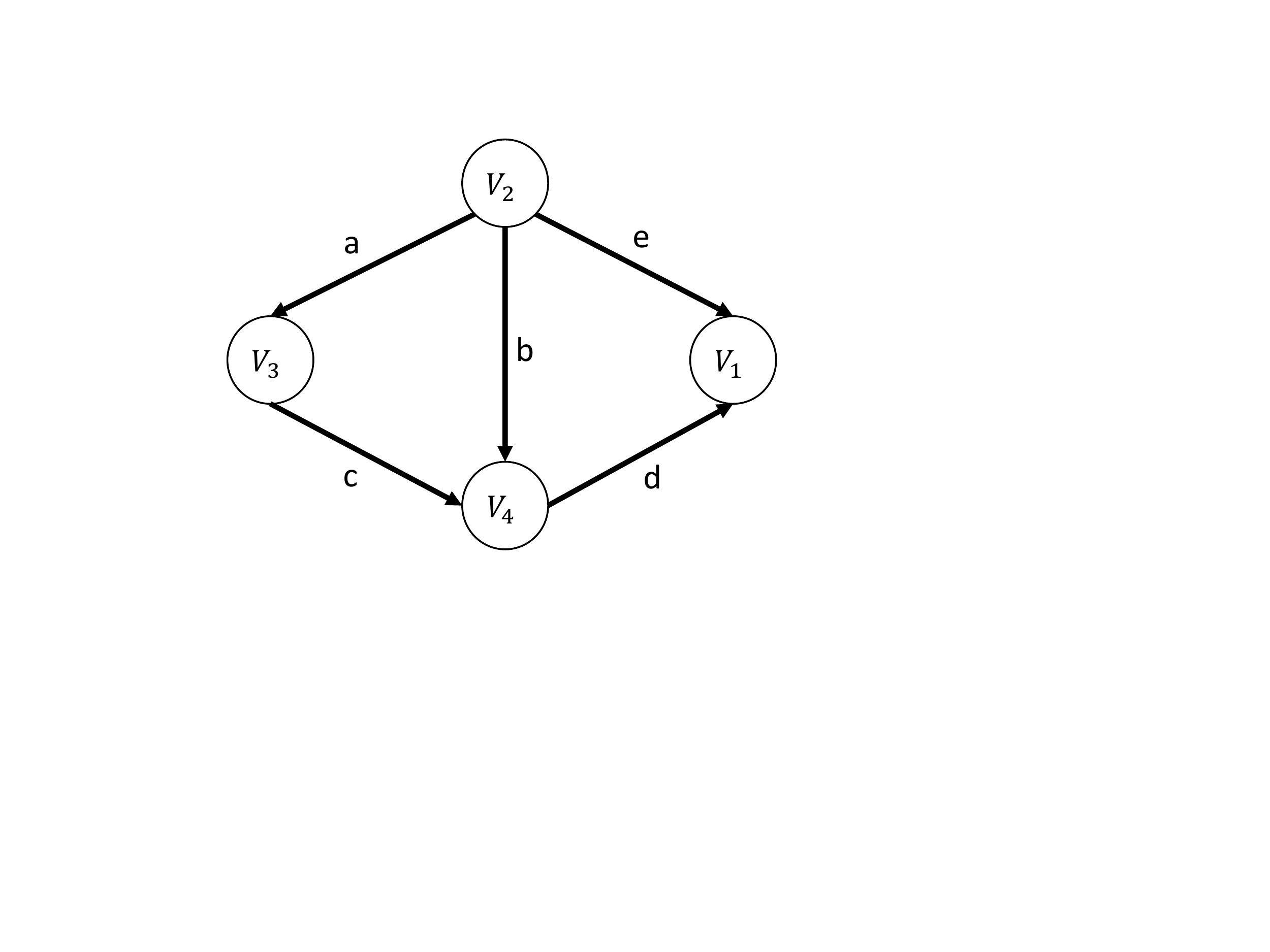}
	\caption{Causal graph of Example \ref{ex:1}. }
	\label{fig:ex1}
\end{figure}

\begin{example}
	In Example \ref{ex:1}, suppose that only variables $V_1$ and $V_2$ are observable. Then, the causal order among observed variables would be: $k_o(1)=2$ and $k_o(2)=1$. Thus, $\mathbf{P}\mathbf{A_{oo}}\mathbf{P}^T$ is a strictly lower triangular matrix where $\mathbf{P}=[0,1;1,0]$. For the latent variables, $k_l(3)=1$ and $k_l(4)=2$. 
\end{example}

In the remainder of this section, we briefly describe LiNGAM algorithm, which is capable of recovering the matrix $\mathbf{A}$ uniquely if all variables in the model are observable and exogenous noises are non-Gaussian  \citep{shimizu2006linear}.
The vector $\mathbf{V}$ in Equation \eqref{eq:model} can be written as a linear combination of exogenous noises as follows:
\begin{equation}
\mathbf{V}=\mathbf{B}\mathbf{N},
\label{eq:3}
\end{equation}
where $\mathbf{B}=(\mathbf{I}-\mathbf{A})^{-1}$. The above equation fits into the standard linear Independent Component Analysis (ICA) framework, where independent non-Gaussian components are all variables in $\mathbf{N}$. 
By utilizing statistical techniques in ICA \citep{hyvarinen2004independent}, matrix $\mathbf{B}$ can be identified up to scaling and permutations of its columns. 
More specifically, the independent components of ICA as well as the estimated $\mathbf{B}$ matrix are not uniquely determined because permuting and rescaling them does not change their mutual independence. 
So without knowledge of the ordering and scaling of the noise terms, the following general ICA model for $\mathbf{V}$ holds:
\begin{equation}
\mathbf{V} = \mathbf{\tilde{B}} \mathbf{\tilde{N}},
\label{eq:4}
\end{equation}
where $\mathbf{\tilde{N}}$ contains independent components and these components (resp. the columns of $\mathbf{\tilde{B}}$) are a permuted and rescaled version of those in $\mathbf{N}$ (resp. the columns of $\mathbf{B}$). In what follows, we use $\mathbf{B}$ for the matrix $\mathbf{B} = (\mathbf{I}-\mathbf{A})^{-1}$ while $\mathbf{\tilde{B}}$ is the mixing matrix for the ICA model, as given in \eqref{eq:4}. Hence $\tilde{\mathbf{B}}$ can be written as:
$$\mathbf{\tilde{B}} = \mathbf{B} \mathbf{P} \mathbf{\Lambda},$$
where $\mathbf{P}$ is a permutation matrix and $\mathbf{\Lambda}$ is a diagonal scaling matrix.
Yet the corresponding causal model, represented by $\mathbf{A}$, can be uniquely identified because of its acyclicity constraint. In particular, the inverse of $\mathbf{B}$ can be converted uniquely to a lower triangular matrix having all-ones on its diagonal by some scaling and permutation of the rows.

\section{Identifying Causal Orders among Observed Variables}
\label{section:path}

Since the graph with adjacency matrix $\mathbf{A}$ is acyclic, there exists an integer $d$ such that $\mathbf{A}^d=0$. Thus, we can rewrite $\mathbf{B}$ in the following form:
\begin{equation}
\mathbf{B}=(\mathbf{I}-\mathbf{A})^{-1}=\sum_{k=0}^{d-1} \mathbf{A}^k.
\end{equation}

It can be seen that there exists a casual path of length $k$ from the exogenous noise of variable $V_i$ to variable $V_j$ if entry $(j,i)$ of matrix $\mathbf{A}^k$ is nonzero. We define $[\mathbf{B}]_{j,i}$ as the total causal effect of variable $V_i$ on variable $V_j$. 

\begin{assumption}
	(Faithfulness assumption) The total causal effect from variable $V_i$ to $V_j$ is nonzero if there is a causal path from $V_i$ to $V_j$. Thus, we have: $[\mathbf{B}]_{j,i}\neq 0$ if $V_i \rightsquigarrow V_j$. 
	\label{assu:fauith}
\end{assumption}

In the following lemma, we list two consequences of the faithfulness assumption that are immediate from the definition. 


\begin{lemma}
	Under the faithfulness assumptions, for any two observed variables $V_i$ and $V_j$, $1\leq i,j\leq p_o$, the following holds:
	\\ (i) Suppose that $V_i \rightsquigarrow V_j$. If $[\mathbf{B}]_{i,k}\neq 0$ for some $k\neq j$, then $[\mathbf{B}]_{j,k}\neq 0 $.
	\\ (ii) If there is no causal path between $V_i$ and $V_j$, then $[\mathbf{B}]_{i,j}=0$ and $[\mathbf{B}]_{j,i}=0$.
	\label{lemma:faithfulness} 
\end{lemma}
%

Based on Equation \eqref{eq:main}, we can write $\mathbf{V_o}$ in terms of $\mathbf{N_o}$ and $\mathbf{N_l}$ as follows
\begin{equation}
\mathbf{V_o}=(\mathbf{I}-\mathbf{D})^{-1}\mathbf{N_o}+ (\mathbf{I}-\mathbf{D})^{-1}\mathbf{A_{ol}}(\mathbf{I}-\mathbf{A_{ll}})^{-1}\mathbf{N_l},
\label{eq:mainp}
\end{equation}
where $\mathbf{D}=\mathbf{A_{oo}}+\mathbf{A_{ol}}(\mathbf{I}-\mathbf{A_{ll}})^{-1}\mathbf{A_{lo}}$. Let $\mathbf{B_o}:=(\mathbf{I}-\mathbf{D})^{-1}$, $\mathbf{B_l}:=(\mathbf{I}-\mathbf{D})^{-1}\mathbf{A_{ol}} (\mathbf{I}-\mathbf{A_{ll}})^{-1}$, and $\mathbf{N}:=[\mathbf{N_o};\mathbf{N_l}]$. Thus,  $\mathbf{V_o}=\mathbf{B'}\mathbf{N}$ where  $\mathbf{B'}:=[\mathbf{B_o},\mathbf{B_l}]$. This equation fits into a linear over-complete ICA where the exogenous noises are non-Gaussian and the number of observed variables is less than the number of variables in the system.
The following proposition asserts when the columns of matrix $\mathbf{B'}$ are still identifiable up to some permutations and scaling. 

\begin{definition}
	(Reducibility of a matrix) A matrix is reducible if two of its columns are linearly dependent.
\end{definition}

\begin{proposition}(\citep{eriksson2004identifiability}, Theorem 3)
	In the linear over-completer ICA problem, the columns of mixing matrix can be identified up to some scaling and permutation if it is not reducible. 
	\label{prop:id}
\end{proposition}

\begin{lemma}
	The columns of $\mathbf{B'}$ corresponding to any two observed variables are linearly independent.
	\label{lemma:obsred}
\end{lemma}
\begin{proof}
	Consider any two observed variables $V_i$ and $V_j$. We know that $[\mathbf{B'}]_{i,i}$ and $[\mathbf{B'}]_{j,j}$ are non-zero. Furthermore, $\mathbf{B'}$ is a sub-matrix of $\mathbf{B}$. Hence, based on Lemma \ref{lemma:faithfulness} ($ii$), if there is no causal path between $V_i$ and $V_j$, we have:  $[\mathbf{B'}]_{i,j}=0$ and $[\mathbf{B'}]_{j,i}=0$. Thus, $[\mathbf{B'}]_{:,i}$ and $[\mathbf{B'}]_{:,j}$ are not linearly dependent. Furthermore, if one of the variable is the ancestor of the another one, let say $V_i\in anc(V_j)$, according to Lemma \ref{lemma:faithfulness} ($i$), $[\mathbf{B'}]_{j,i}\neq 0$ while $[\mathbf{B'}]_{i,j}=0$. Thus, $[\mathbf{B'}]_{:,i}$ and $[\mathbf{B'}]_{:,j}$ are also not linearly dependent in this case and the proof is complete.
\end{proof}

Although columns of $\mathbf{B'}$ corresponding to the observed variables are pairwise linearly independent, a column corresponding to a latent variable $V_i$ might be linearly dependent on a column corresponding to an observed or latent variable $V_j$ (see Example \ref{ex:3new2}).  In that case, we can remove the column $[\mathbf{B'}]_{:,i}$ and $N_i$ from matrix $\mathbf{B'}$ and vector $\mathbf{N}$, respectively and replace $N_j$ by $N_j+\alpha N_i$ where $\alpha$ is a constant such that $[\mathbf{B'}]_{:,i}=\alpha[\mathbf{B'}]_{:,j}$.
%
We can continue this process until all the remaining columns are pairwise linearly independent. 
Let $\mathbf{B''}$ and $\mathbf{N''}$ be the resulting mixing matrix and exogenous noise vector, respectively. According to Lemma \ref{lemma:obsred}, all the columns of $\mathbf{B'}$ corresponding to observed variables are in $\mathbf{B''}$. We utilize the matrix $\mathbf{B''}$ to recover a causal order among the observed variables.


Since the matrix $\mathbf{B''}$ is not reducible, its column can be identified up to some scaling and permutation according to Proposition \ref{prop:id}. Let $\mathbf{\tilde{B}''}$ be the recovered matrix containing columns of $\mathbf{B''}$.
Consider any two observed variables $V_i$ and $V_j$, i.e., $1\leq i,j\leq p_o$.   
We extract two rows of $\mathbf{\tilde{B}''}$ corresponding to variables $V_i$ and $V_j$. 
%
Let $n_{0*}$ be the number of columns in $[\mathbf{\tilde{B}''}_{i,:};\mathbf{\tilde{B}''}_{j,:}]$  whose first entries are zero but second entries are nonzero. Similarly, let $n_{*0}$ be the number of columns that their first entries are nonzero but their second entries are zero. The following lemma asserts that the existence of a causal path between $V_i$ and $V_j$ can be checked from $n_{0*}$ and $n_{*0}$ (or equivalently, $\mathbf{\tilde{B}}''$).

\begin{lemma}
	Under the faithfulness assumption, the existence of a causal path between any two observed variable can be inferred from matrix $\mathbf{\tilde{B}}''$.
\end{lemma} 
\noindent
{\bf Proof}.
First, we show that if $V_i \rightsquigarrow V_j$, then $n_{0*}>0$ and $n_{*0}=0$. We know that the matrix $[\mathbf{\tilde{B}''}_{i,:};\mathbf{\tilde{B}''}_{j,:}]$ can be converted to $[\mathbf{B}''_{i,:};\mathbf{B}''_{j,:}]$ by some permutation and scaling of its columns. Moreover,
 $\mathbf{B''}$ contains some of the columns of $\mathbf{B'}$ including all the columns corresponding to the observed variables.
 Thus, from Lemma \ref{lemma:faithfulness}, we know that if $[\mathbf{B''}]_{i,k}\neq 0$ for any $k\neq j$, then $[\mathbf{B''}]_{j,k}\neq 0$. Moreover, we have: $[\mathbf{B''}]_{j,j}\neq 0$ and $[\mathbf{B''}]_{i,j}=0$. Hence, we can conclude that: $n_{0*}>0$ and $n_{*0}=0$. 

If $n_{0*}>0$ and $n_{*0}=0$, then $V_i \rightsquigarrow V_j$. By contradiction, suppose that there is no causal path between $V_i$ and $V_i$ or $V_j \rightsquigarrow V_i$. The second case ($V_j \rightsquigarrow V_i$) does not happen due to what we just proved. Furthermore, from Lemma \ref{lemma:faithfulness}, we know that $[\mathbf{B''}]_{i,i}\neq 0$, $[\mathbf{B''}]_{i,j}= 0$. Therefore, $n_{*0}>0$ which is in contradiction with our assumption. Hence, we can conclude that $n_{0*}>0$ and $n_{*0}=0$ if and only if $V_i \rightsquigarrow V_j$.  
\hfill \BlackBox
\\


We can construct an auxiliary directed graph whose vertices are the observed variables and a directed edge exists from $V_i$ to $V_j$ if $V_i\rightsquigarrow V_j$ (which we can infer from $n_{*0}$ and $n_{0*}$). Any causal order over the auxiliary graph is a correct causal order among the observed variables $\mathbf{V_o}$.  
%

\begin{example}
	Consider the causal graph in Figure \ref{fig:ex3new2}. Suppose that variables $V_3$ and $V_4$ are latent. The matrix $\mathbf{B'}$ would be:
	\begin{equation*}
	\begin{bmatrix}
		1 & 0 & 0 & a\\
		d & 1 & e & c+ad+be
	\end{bmatrix}.
	\end{equation*}
	We can remove the third column from $\mathbf{B'}$ and update the vector $\mathbf{N}$ to $[N_1;N_2+eN_3;N_4]$. Thus, the matrix $\mathbf{B''}$ is equal to:
	\begin{equation*}
	\begin{bmatrix}
	1 & 0 & a\\
	d & 1 & c+ad+be
	\end{bmatrix},
	\end{equation*}
	which is not reducible. Without loss of generality, assume that the recovered matrix $\mathbf{\tilde{B}}''$ is equal to $\mathbf{B}''$. Therefore, $n_{0*}=1$ and $n_{*0}=0$. Hence, we can infer that there is a causal path from $V_1$ to $V_2$.
	\label{ex:3new2}
\end{example}

\begin{figure}
	\centering
	\includegraphics[width=0.33\textwidth]{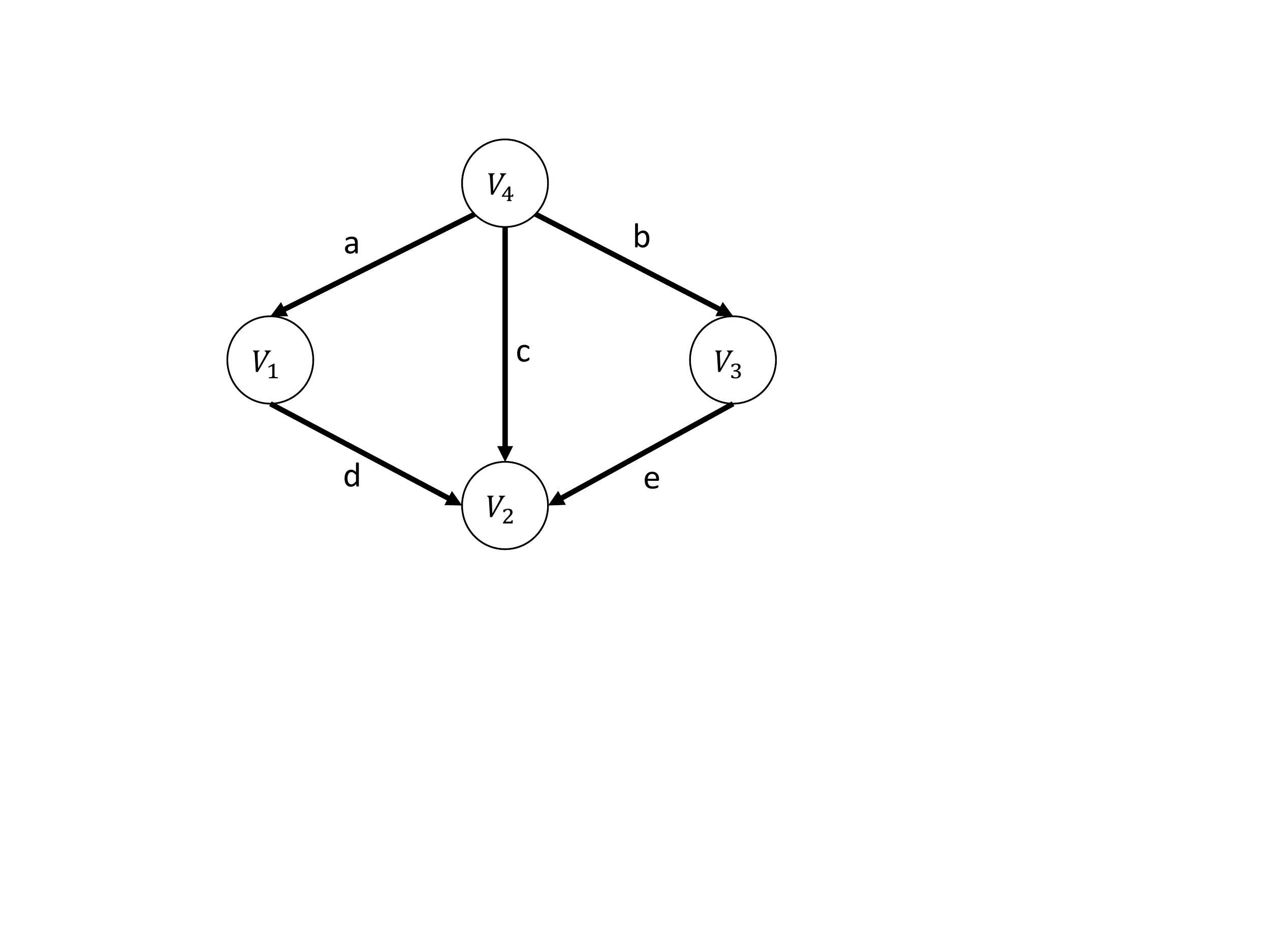}
	\caption{Causal graph of Example \ref{ex:3new2}. }
	\label{fig:ex3new2}
\end{figure}

\subsection*{Recovering the Number of Variables in the System}
 According to Proposition \ref{prop:id}, the number of variables in the system can be recovered if and only if the matrix $\mathbf{B'}$ is not reducible. Furthermore, Equation \eqref{eq:mainp} implies that matrix $\mathbf{B'}$ is not reducible if and only if the columns of the following matrix are not linearly independent: $[\mathbf{I}_{p_o\times p_o}|\mathbf{A_{ol}}(\mathbf{I}-\mathbf{A_{ll}})^{-1}]$. In the rest of this section, we will present equivalent necessary and sufficient graphical conditions under which the number of variables in the systems can be uniquely identified.
But before that, we present a simple example where $[\mathbf{I}_{p_o\times p_o}|\mathbf{A_{ol}}(\mathbf{I}-\mathbf{A_{ll}})^{-1}]$ is reducible and give a graphical interpretation of it.

\begin{example}
	Consider a linear SEM with three variables $V_1,V_2,$ and $V_3$ where $V_3=N_3$, $V_1=\alpha V_3+N_1$, and $V_2=\beta V_1+N_2$. Thus, the corresponding causal graph would be: $V_3\rightarrow V_1 \rightarrow V_2$. Suppose that $V_3$ is the only latent variable. Hence, $\mathbf{A_{ll}}=0$, $\mathbf{A_{ol}}=[\alpha;0]$, and $\mathbf{A_{ol}}(\mathbf{I}-\mathbf{A_{ll}})^{-1}=[\alpha;0]$ which is linearly dependent on the first column of $\mathbf{I}$. In fact, latent variable $V_3$ can be absorbed in variable $V_1$ by changing the exogenous noise of $V_1$ from $N_1$ to $N_1+\alpha N_3$. Thus, the number of variables in this model cannot be identified uniquely in this model.
	\label{ex:3new}
\end{example}

\begin{definition}(Absorbing)
	Variable $V_i$ is said to be absorbed in variable $V_j$ if the exogenous noise of $V_i$ is set to zero $N_i\leftarrow 0$, and the exogenous noise of $V_j$ is replaced by $N_j\leftarrow N_j+[\mathbf{B}]_{j,i}N_i$. We define absorbing a variable in $\emptyset$ by setting its exogenous noise to zero.
\end{definition}

\begin{definition}(Absorbablity)
	Let $P'_{V_o}$ be the joint distribution of the observed variables after absorbing $V_i$ in $V_j$. We say $V_i$ is absorbable in $V_j$ if $P'_{V_o}=P_{V_o}$.
\end{definition}

 
 The following theorem characterizes the graphical conditions where a latent variable is absorbable. The proof of theorem is given in Appendix A.

\begin{theorem}~\\
	(a) A latent variable is absorbable in $\emptyset$ if and only if it has no observable descendant.\\
	(b) A latent variable $V_j$ is absorbable in variable $V_i$ (observed or latent), if and only if all paths from $V_j$ to its observable
	descendants go through $V_i$.
	\label{lemma:absorb}
\end{theorem}

\begin{example}
	Consider a linear SEM  with corresponding causal graph in Figure \ref{fig:ex3} where $V_1$ and $V_2$ are the only observed variables. $V_7$ satisfies condition (a) and its exogenous noise can be set to zero. Furthermore, $V_3$ and $V_4$ satisfy condition (b) with respect to $V_5$  and they can be absorbed in $V_5$ by setting the exogenous noise of $V_5$ to $N_5+(\alpha\gamma+\beta)N_3+\gamma N_4$.  Finally, $V_6$ satisfies condition (b) and it can be absorbed in $V_2$. Note that $V_8$ and $V_5$ cannot be absorbed in $V_1$ or $V_2$.
	\begin{figure}
		\centering
		\includegraphics[width=0.3\textwidth]{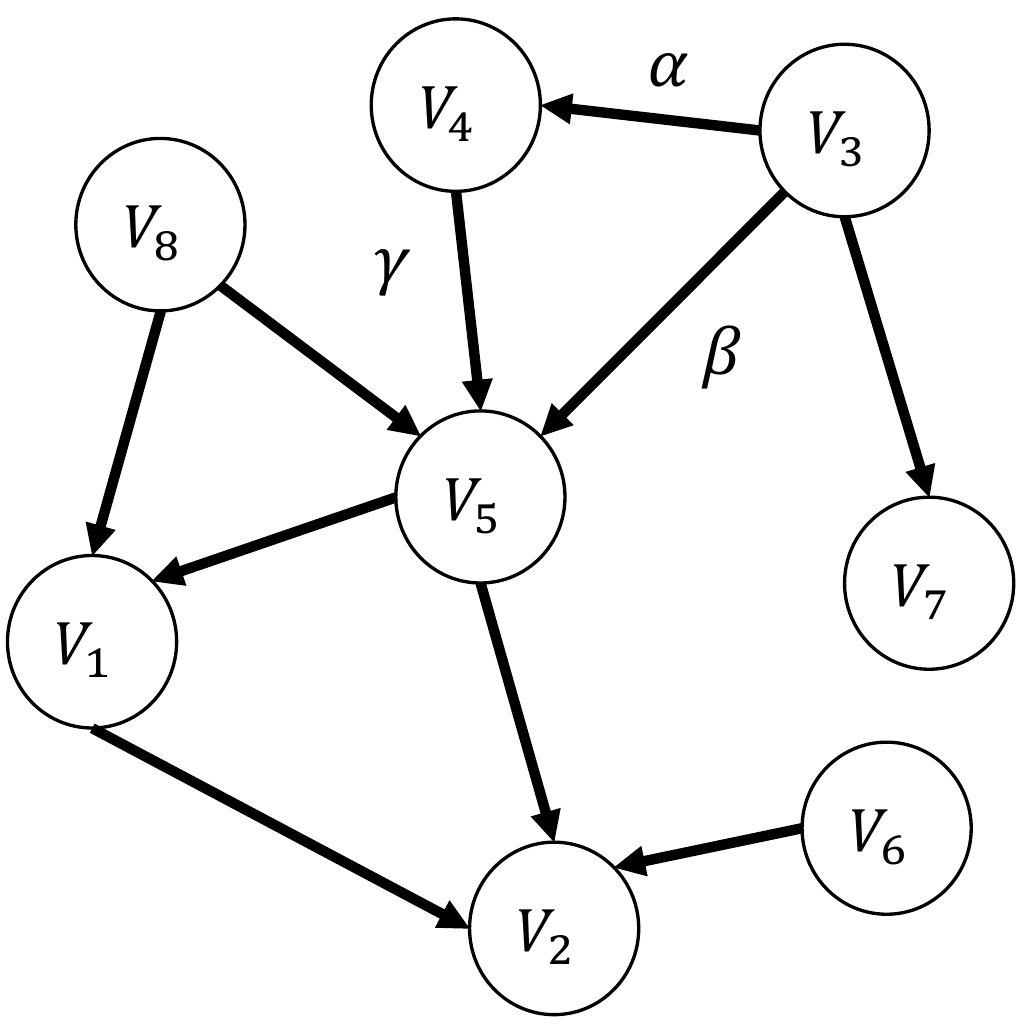}
		\caption{Causal graph of Example \ref{ex:3}. $V_1$ and $V_2$ are the only observed variables. }
		\label{fig:ex3}
	\end{figure} 
	\label{ex:3}
\end{example}


\begin{definition}
	We say a causal graph is minimal if none of its variables are absorbable.	
\end{definition}

Based on above definition, a causal graph is minimal if none of the latent variables satisfy the conditions in Theorem \ref{lemma:absorb}.  
We borrowed the terminology of minimal causal graphs from \cite{pearl1988probabilistic} for polytree causal structures. In \citep{pearl1988probabilistic}, a casual graph is called minimal if it has no redundant latent variables in the sense that the joint distribution without latent variables  remains a connected tree. Later, \cite{etesami2016learning} showed that in minimal latent directed information polytrees, each node has at least two children. The following lemma asserts that the same argument holds true for the non-absorbable latent variables in our setting. The proof of lemma is given in Appendix B.

\begin{lemma}
	A latent variable is non-absorbable if it has at least two non-absorbable children. 
	\label{lemma:minimal}
\end{lemma} 

Next theorem gives necessary and sufficient graphical conditions for non-reduciblity of the matrix $\mathbf{B'}$. The proof of theorem is given in Appendix C.

\begin{theorem}
	The matrix $\mathbf{B'}$ is not reducible almost surely if and only if the corresponding causal graph $G$ is minimal.
	\label{cond:eq}
\end{theorem}

\begin{corollary}
	Under faithfulness assumption and non-Gaussianity of exogenous noises, the number of variables in the system is identifiable almost surely if the corresponding graph is minimal.	
\end{corollary}
\noindent
{\bf Proof}.
	Based on Theorem \ref{cond:eq}, we know that the matrix $\mathbf{B'}$ is not reducible almost surely if the corresponding causal graph $G$ is minimal. Furthermore, according to Proposition \ref{prop:id}, the number of variables in the systems is identifiable if matrix $\mathbf{B'}$ is not reducible. This completes the proof.
\hfill\BlackBox\\

\section{Identifying Total Causal Effects among Observed Variables}

In this section, first, we will show by an example that total causal effects among observed variables cannot be identified uniquely under the faithfulness assumption and non-Gaussianity of exogenous noises\footnote{This example has also been studied in \citep{hoyer2008estimation}.}. However, we can obtain all the possible solutions. Furthermore, under some additional assumptions on linear SEM, we show that one can uniquely identify total causal effects among observed variables. 

\subsection{Example of non-Uniqueness of Total Causal Effects}
\label{sec:example}

Consider causal graph in Figure \ref{fig:example} where $V_i$ and $V_j$ are observed variables and $V_k$ is latent variable. The direct causal effects from $V_k$ to $V_i$, from $V_k$ to $V_j$, and from $V_i$ to $V_i$ are $\alpha$, $\gamma$, and $\beta$, respectively.
We can write $V_i$ and $V_j$ based on the exogenous noises of their ancestors as follows:
\begin{align}
\begin{split}
V_i &= \alpha N_k+ N_i,\\
V_j &= \beta N_i +(\alpha\beta+\gamma) N_k + N_j.
\label{eq:examplesimple}
\end{split}
\end{align}

Now, we construct a second causal graph depicted in Figure \ref{fig:example} where the exogenous noises of variables $V_i$ and $V_k$ are changed to $\alpha N_k$ and $N_i$, respectively. Furthermore, we set the direct causal effects from $V_k$ to $V_i$, from $V_k$ to $V_j$, and from $V_i$ to $V_j$ to $1$, $-\gamma/\alpha$, and $\beta+(\gamma/\alpha)$, respectively.
It can be seen that equations in \eqref{eq:examplesimple} do not change while the direct causal effect from $V_i$ to $V_j$ becomes $\beta+(\gamma/\alpha)$ in the second causal graph. Thus, we cannot identify causal effect from $V_i$ to $V_j$ merely by observational data from $V_i$ and $V_j$. In Appendix D, we extend this example to the case where there might be multiple latent variables on the path from $V_k$ to $V_i$ and $V_j$, and from $V_i$ to $V_j$.

\begin{figure}
	\centering
	\includegraphics[width=0.6\textwidth]{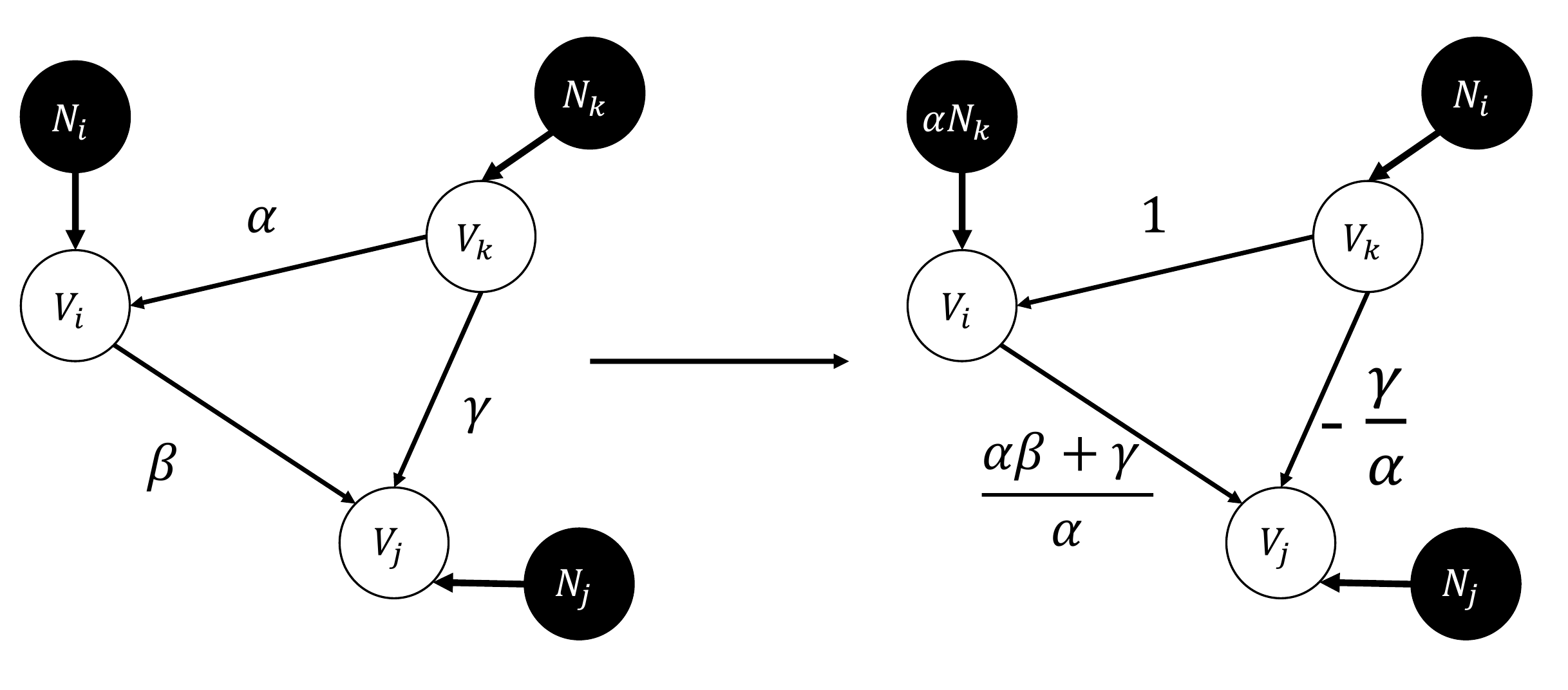}
	\caption{An example of non-identifiability of causal effects from observed variable $V_i$ to observed variable $V_j$. }\label{fig:example}
\end{figure}

The above example shows that causal effects may not be identified even by assuming non-Gaussianity of exogenous noises if we have some latent variables in the system. In the following, we first show that the set of all possible total causal effects can be identified. Afterwards, we will present a set of structural conditions under which we can uniquely identify total causal effects among observed variables.

\subsection{Identifying the Set of All Possible Total Causal Effects}


Since the subgraph corresponding to $\mathbf{A_{ll}}$ is a DAG, there exists an integer $d_l$ such that $\mathbf{A_{ll}}^{d_l}=0$. Hence, we can rewrite matrix $\mathbf{D}$ given in \eqref{eq:mainp} as follows
\begin{equation}
\mathbf{D}=\mathbf{A_{oo}}+\sum_{k=0}^{d_l-1}\mathbf{A_{ol}}\mathbf{A_{ll}}^k\mathbf{A_{lo}}.
\end{equation}

\begin{lemma}
	Matrix $\mathbf{D}$ in \eqref{eq:mainp} can be converted to a strictly lower triangular matrix by permuting columns and rows simultaneously based on the causal order $k_o$.
	\label{lemma:1}
\end{lemma}
\noindent
{\bf Proof}.
	Let $\mathbf{P}$ be the permutation matrix corresponding the causal order $k_o$. We want to show that $\mathbf{P}\mathbf{D}\mathbf{P}^T$ is strictly lower triangular. It suffices to prove  $\mathbf{P}\mathbf{A_{ol}}\mathbf{A_{ll}}^k\mathbf{A_{lo}}\mathbf{P}^T$ is strictly lower triangular for any $0\leq k\leq d_l-1$. Suppose that there exists a nonzero entry, $(i,j)$, in $\mathbf{P}\mathbf{A_{ol}}\mathbf{A_{ll}}^k\mathbf{A_{lo}}\mathbf{P}^T$ where $j\geq i$. Then, there should be a directed path from observed variable $V_{k^{-1}_o(j)}$ to $V_{k_o^{-1}(i)}$ of length $k+2$ through latent variables in the causal graph where $k_o^{-1}(i)$ is the index of an observed variable whose order is $i$ in the causal order $k_o$. This means variable $V_{k_o^{-1}(j)}$ should come before variable $V_{k_o^{-1}(i)}$ in any causal order. But this violates the causal order $k_o$. 
\hfill\BlackBox\\

Previously, we showed that existence of a causal path between any two observed variables $V_i$ and $V_j$ can be determined by performing over-complete ICA. Let $des_o(V_i)$ be the set of all observed descendants of $V_i$, i.e., $des_o(V_i)=\{V_j|V_i\rightsquigarrow V_j, 1\leq j \leq p_o\}$. We will utilize $des_o(V_i)$'s to enumerate all possible total causal effects among the observed variables.

\begin{remark}
	From Lemma \ref{lemma:obsred}, we have: $des_o(V_i)\neq des_o(V_j)$ for any $1\leq i,j \leq p_o$. 
	\label{remark:dec}
\end{remark}

As we discussed in Section \ref{section:path}, under non-Gaussianity of exogenous noises, the columns of $\mathbf{B''}$ can be determined up to some scalings and permutations by solving an overcomplete ICA problem. Let $p_r$ be the number of columns of $\mathbf{B''}$. Furthermore, without loss of generality, assume that variables $V_{p_o+1},V_{p_o+1},\cdots, V_{p_r}$ are the latent variables in the system whose corresponding columns remain in $\mathbf{B''}$.

\begin{theorem}
	Let $r_i:=|\{j:des_o(V_i)=des_o(V_j), 1 \leq j \leq p_r\}|$, for any $1\leq i \leq p_o$. Under faithfulness assumption and non-Gaussianity of exogenous noises, the number of all possible $\mathbf{D}$'s that can generate the same distribution for $\mathbf{V_o}$ according to \eqref{eq:main}, is equal to $\Pi_{i=1}^{p_o} r_i$.
	\label{th:identifycausal}
\end{theorem}

\noindent
{\bf Proof}.
	 According to Proposition \ref{prop:id}, under non-Gaussianity of exogenous noises, the columns of $\mathbf{B''}$ can be determined up to some scalings and permutations by solving an overcomplete ICA problem.
	 Furthermore, for the column corresponding to the noise $N_i$, $1\leq i\leq p_o$, we have $r_i$ possible candidates with the same set of indicies of non-zero entries where all of them are pairwise linearly independent.  Let $\mathbf{B_o'}$ be a $p_o\times p_o$ matrix by selecting one of the candidates for each column corresponding to noise $N_i$, $1\leq i \leq p_o$. Thus, we have $\Pi_{i=1}^{p_o} r_i$  possible matrices\footnote{Please note that diagonal entries of $\mathbf{B_o'}$ should be equal to one. Otherwise we can normalize each column to its on-diagonal entry.}. Now, for each $\mathbf{B_o'}$, we just need to show that there exists an assignment for $\mathbf{A_{oo}}$, $\mathbf{A_{lo}}$, $\mathbf{A_{ol}}$, and $\mathbf{A_{ll}}$ such that they satisfy \eqref{eq:mainp} and  $\mathbf{A_{oo}}$ and  $\mathbf{A_{ll}}$ can be converted to strictly lower triangular matrices with some simultaneous permutations of columns and rows.
	
	Let $\mathbf{A_{lo}}=\mathbf{0}_{p_l\times p_o}$ and  $\mathbf{A_{ll}}=\mathbf{0}_{p_l\times p_l}$. Assume that $\mathbf{B'_l}$ consists of the remaining columns which are not in $\mathbf{B_o'}$. We also add columns corresponding to latent absorbed variables to $\mathbf{B_l'}$.  Now, we set $\mathbf{A_{oo}}$ and $\mathbf{A_{ol}}$ to $\mathbf{I}-\mathbf{B_o'}^{-1}$ and $\mathbf{B_o'}^{-1}\mathbf{B'_l}$, respectively. By these assignments, the proposed matrix $\mathbf{A}=[\mathbf{A_{oo}},\mathbf{A_{ol}};\mathbf{A_{lo}},\mathbf{A_{ll}}]$ satisfies in \eqref{eq:mainp}. Thus, we just need to show that $\mathbf{I}-\mathbf{B_o'}^{-1}$ can be converted to a strictly lower triangular matrix by some permutations. To do so, first note that from Lemma \ref{lemma:1}, we know that matrix $\mathbf{D}$ can be converted to a strictly lower triangular matrix by a permutation matrix $\mathbf{P}$. Furthermore, based on this property of matrix $\mathbf{D}$, we have: $\mathbf{D}^{p_o}=\mathbf{0}$. Thus, we can write:
	\begin{equation*}
	\mathbf{P}(\mathbf{I}-\mathbf{D})^{-1}\mathbf{P}^T=\sum_{k=0}^{p_o-1} \mathbf{P}\mathbf{D}^k\mathbf{P}^T =\sum_{k=0}^{p_o-1} (\mathbf{P}\mathbf{D}\mathbf{P}^T)^k.
	\end{equation*}
	
	Since the matrix $(\mathbf{P}\mathbf{D}\mathbf{P}^T)^k$ is a lower triangular matrix for any $k\geq 0$, $(\mathbf{I}-\mathbf{D})^{-1}$ can be converted to a lower triangular matrix by permutation matrix $\mathbf{P}$. Furthermore, the set of nonzero entries of $\mathbf{B'_o}$ is the same as the one of $(\mathbf{I}-\mathbf{D})^{-1}$. Thus, $\mathbf{P}\mathbf{B_o'}\mathbf{P}^T$ is also a lower triangular matrix where all diagonal elements of it are equal to one. Hence, we can write $\mathbf{B_o'}$ in the form of $\mathbf{B_o'}=\mathbf{I}+\mathbf{B_o''}$ where $\mathbf{P}\mathbf{B_o''}\mathbf{P}^T$ is a strictly lower triangular matrix. Therefore, we have:
	\begin{equation}
	\mathbf{P}(\mathbf{I}-\mathbf{B_o'}^{-1})\mathbf{P}^T=\mathbf{P}(\mathbf{I}-\sum_{k=0}^{p_o-1} (-1)^k \mathbf{B_o''}^k)\mathbf{P}^T=\mathbf{P}(\sum_{k=1}^{p_o-1} (-1)^{k+1} \mathbf{B_o''}^k)\mathbf{P}^T,
	\end{equation}
	where the last term shows that $\mathbf{I}-\mathbf{B_o'}^{-1}$ can be converted to a strictly lower triangular matrix and the proof is complete. 
	%
\hfill\BlackBox\\

Comparing our results with \citep{hoyer2008estimation}, we can obtain all sets $des_o(V_i)$'s and determine which columns can be selected as corresponding columns of observed variables in $O(p_o^2p_r)$ and then enumerate all the possible total causal effects while the proposed algorithm in \citep{hoyer2008estimation} requires to search a space of ${p_r \choose p_o}$ different possible choices. Moreover, we can identify a causal order uniquely with the same time complexity by utilizing the method proposed in Section \ref{section:path}.

\subsection{Unique Identification of Causal Effects under Structural Conditions}
\label{sec:UniIde}

Based on Theorem \ref{th:identifycausal}, in this part, we propose a method to identify total causal effects uniquely under some structural conditions.

\begin{assumption}
	Assume that for any observed variables $V_i$ and any latent variable $V_k$, we have: $des_o(V_k)\neq des_o(V_i) $.
	\label{assu:dec}
\end{assumption}

Assumption \ref{assu:dec} is a very natural condition that one expects to hold for unique identifiability of causal effects. This is because if Assumption \ref{assu:dec} fails, then based on Theorem \ref{th:identifycausal}, there are multiple sets of total causal effects that are compatible with the observed data. 




\begin{theorem}
	Under Assumptions \ref{assu:fauith}-\ref{assu:dec}, and non-Gaussianity of exogenous noises, the total causal effect between any two observed variables can be identified uniquely. 
	\label{theorem:causal_strength}
\end{theorem}
\noindent
{\bf Proof}.
Let matrix $[\mathbf{\tilde{B}}'']_{p_o\times p_r}$ be the output of over-complete ICA problem whose columns are the columns in matrix $\mathbf{B''}$. We define $I_i$ as the the set of indices of nonzero entries of column $\mathbf{\tilde{B}}''_{:,i}$, i.e. $I_i=\{k|[\mathbf{\tilde{B}}''_{:,i}]_{k}\neq 0\}$. We know that $I_i=des_o(V_j)$ if $\mathbf{\tilde{B}}''_{:,i}$ corresponds to the observed variable $V_j$. Moreover, under Assumption \ref{assu:dec}, any observed variable $V_i$ and any variable $V_j$ (observed or latent) have different sets $des_o(V_i)$ and $des_o(V_j)$. Thus, each set $I_i$ is just equal to one of $des_o(V_i)$'s, let say $des_o(V_j)$. The column $\mathbf{\tilde{B}}''_{:,i}$ normalized to $[\mathbf{\tilde{B}}''_{:,i}]_{j}$ shows the total causal effects from variable $j$ to other observed variables. 
\hfill\BlackBox\\

%
%

\begin{algorithm}[!t]
	\caption{}\label{alg:causalstrengths}
	\begin{algorithmic}[1]
		\STATE {\bf Input:} Collection of the sets $des_o(V_i),1\leq i \leq p_o$.
		\STATE Run an over-complete ICA algorithm over observed variables $\mathbf{V_o}$ and obtain matrix $\mathbf{\tilde{B}''}$.
		\FOR{$i=1:p_r$}
		\STATE $I_i=\{k|[\mathbf{\tilde{B}''}_{:,i}]_k\neq 0\}$
		\FOR{$j=1:p_o$}
		\IF{$I_i=des_o(V_j)$}
		\STATE $[\mathbf{\hat{B}_o}]_{:,j}=\mathbf{\tilde{B}''}_{:,i}/[\mathbf{\tilde{B}''}_{:,i}]_j$
		\ENDIF
		\ENDFOR
		\ENDFOR
		\STATE {\bf Output:} $\mathbf{\hat{B}_o}$
	\end{algorithmic} 
\end{algorithm}
The description of the proposed solution in Theorem \ref{theorem:causal_strength} is given in Algorithm \ref{alg:causalstrengths}. It is noteworthy to mention that the example in Section \ref{sec:example} (given in Figure \ref{fig:example}) violates the conditions in Theorem \ref{theorem:causal_strength} since we have $des_o(V_k)=des_o(V_i)$. We have shown for this example that the causal effect from $V_i$ to $V_j$ cannot be identified uniquely.

\section{Experiments}
In this section, we first evaluate the performance of the proposed method in recovering causal orders from synthetic data, generated according to the causal graph in Figure \ref{fig:exampleintro}. Our experiments show that the proposed method returns a correct causal order while, as we have discussed in Introduction section, the previous methods \citep{entner2010discovering,shimizu2011directlingam} cannot identify the causal order. We also consider another causal graph which satisfies Assumption \ref{assu:dec} and demonstrate that the proposed method can return the correct causal effects. Next, we evaluate the performance of the proposed method for different number of variables in the system.
Afterwards, for real data, we consider the daily closing prices of four world stock indicies and check the existence of causal paths between any two indicies. The results are compatible with common beliefs in economy.   

\subsection{Synthetic data}
First, for the causal graph in Figure \ref{fig:exampleintro}, we generated 1000 samples of observed variables $V_1$ and $V_2$ where nonzero entries of matrix $\mathbf{A}$ is equal to $0.9$. We utilized Reconstruction ICA (RICA) algorithm \citep{le2011ica} to solve the over-complete ICA problem as follows: Let $\mathbf{v_o}$ be a $p_o\times n$ matrix containing observational data where $[v_o]_{i,j}$ is $j$-th sample of variable $V_i$ and $n$ is the number of samples. First, the sample covariance matrix of $\mathbf{v_o}$ is eigen-decomposed, i.e., $1/(n-1)(\mathbf{v_o}-\mathbf{\bar{v}_o})(\mathbf{v_o}-\mathbf{\bar{v}_o})^T=\mathbf{U}\mathbf{\Sigma}\mathbf{U}^T$ where $\mathbf{U}$ is the orthogonal matrix, $\mathbf{\Sigma}$ is a diagonal matrix, and $\mathbf{\bar{v}_o}$ is the sample mean vector. Then, the observed data is pre-whitened as follows: $\mathbf{w}=\mathbf{\Sigma}^{-1/2}\mathbf{U}(\mathbf{v_o}-\mathbf{\bar{v}_o})$. The RICA algorithm tries to find matrix $\mathbf{Z}$ that is the minimizer of the following objective function:
\begin{equation*}
\underset{{\mathbf{Z}}}{\mbox{minimize }} \sum_{i=1}^n\sum_{j=1}^{p_r} g(\mathbf{Z}_{:,j}^T\mathbf{w}_{:,i})+ \frac{\lambda}{n} \sum_{i=1}^n \|\mathbf{Z}\mathbf{Z}^T \mathbf{w}_{:,i}- \mathbf{w}_{:,i}\|_2^2,
\end{equation*}
where parameter $\lambda$ controls the cost of penalty term. We estimated the matrix $\mathbf{\tilde{B}''}$ by $\mathbf{U}\mathbf{\Sigma}^{1/2}\mathbf{Z}^*$ where $\mathbf{Z}^*$ is the optimal solution of the above optimization problem.


In order to estimate the number of columns of $\mathbf{\tilde{B}''}$, we held out 250 of samples for model selection. More specifically, we solved the over-complete ICA problem for different number of columns, evaluated the fitness of each model by computing the objective function of RICA over the hold-out set, and selected the model with minimum cost. In order to check whether an entry is equal to zero, we used the bootstrapping method \citep{efron1994introduction}, which generates $10$ bootstrap samples by sampling with replacement from training data. For each bootstrap sample, we executed RICA algorithm to obtain an estimation of $\mathbf{\tilde{B}''}$. Since in each estimation, columns are in arbitrary permutation, we need to match similar columns in estimations of $\mathbf{\tilde{B}''}$. To do so, in each estimation, we divided all entries of a column by the entry with the maximum absolute value in that column. Then, we picked each column from the estimated mixing matrix, computed its $l_2$ distance from each column of another estimated mixing matrix, and matched to the one with a minimum distance. Afterwards, we used t-test with confidence level of $95\%$ to check whether an entry is equal to zero from the bootstrap samples. An estimation of $\mathbf{\tilde{B}''}$ from a bootstrap sample is given as follows:

\begin{equation}
\nonumber\begin{bmatrix}
-0.0272       &  0.5238 & 1  \\
1    & 1 & 0.8579  
\end{bmatrix}.
\end{equation}

Moreover, experimental results showed the correct support of $\mathbf{\tilde{B}''}$, i.e., \linebreak $[0,1,1;1,1,1]$ can be recovered with merely $10$ bootstrap samples. Thus, there is a causal path from $V_1$ to $V_2$. Furthermore, for the causal graph $V_1\leftarrow V_3 \rightarrow V_2$ in which $V_3$ is only the latent variable, we repeated the same procedure explained above. An estimation of $\mathbf{\tilde{B}''}$ from one of the bootstrap samples is given as follows:

\begin{equation}
\nonumber\begin{bmatrix}
1       &  -0.046 & 0.9838  \\
-0.031    & 1 &  1  
\end{bmatrix}.
\end{equation}

\begin{figure}[t]
	\centering
	\includegraphics[width=0.3\textwidth]{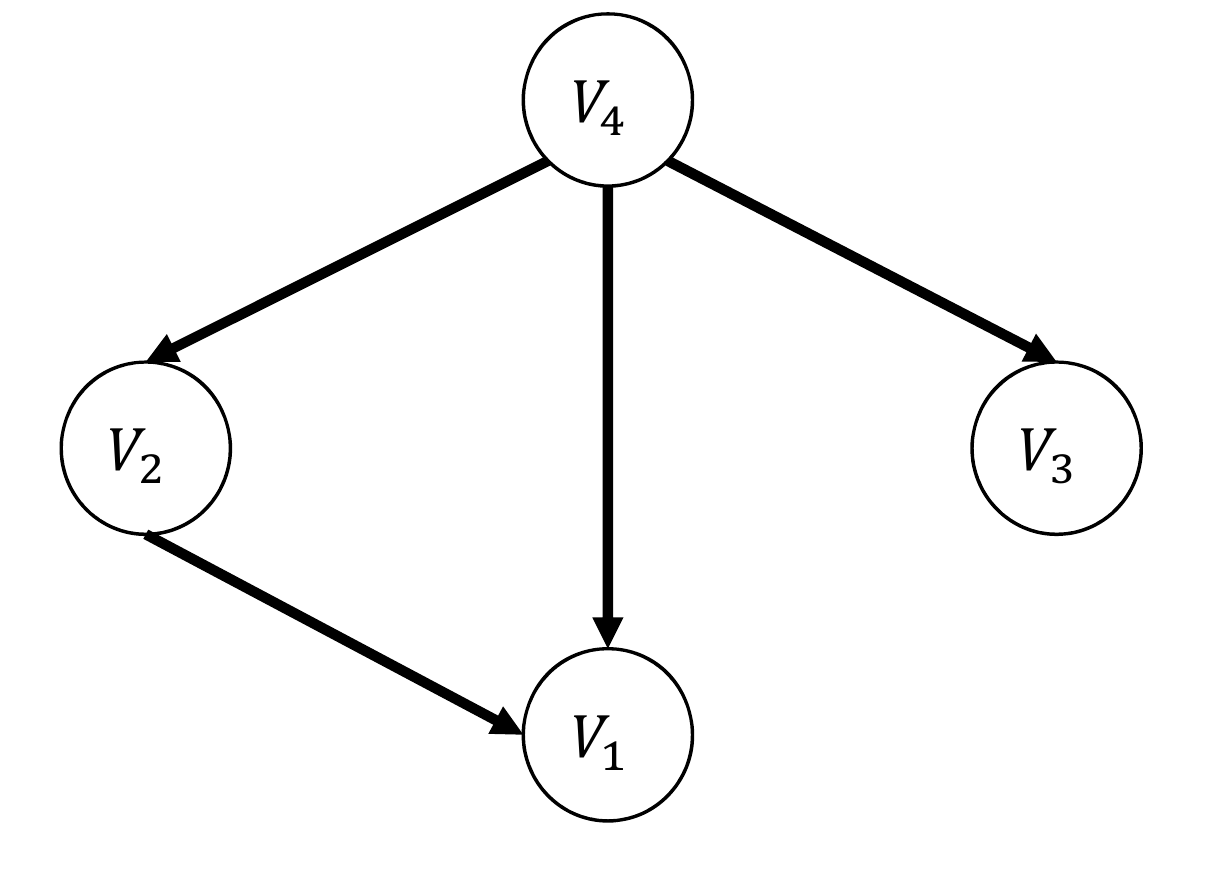}
	\caption{An example of casual graphs satisfying structural conditions. }\label{fig:examplestruc}
\end{figure} 

From experiments, the estimated support of $\mathbf{\tilde{B}''}$ from bootstrap samples would be: $[0,1,1;1,0,1]$. Thus, we can conclude that there is no causal path between $V_1$ and $V_2$.
Next, we considered the causal graph in Figure \ref{fig:examplestruc} where $V_4$ is the only latent variable. The direct causal effects of all directed edges are equal to $0.9$.  An estimation of $\mathbf{\tilde{B}''}$ from one of the bootstrap samples is given as follows:

\begin{equation}
\nonumber\begin{bmatrix}
-0.049      & 0.892 & 1  & 1  \\
-0.024  & 1 & 0.523 & -0.042\\
1  &-0.02 & 0.527 & -0.032
\end{bmatrix}.
\end{equation}

Thus, we can imply that there is only a causal path from $V_2$ to $V_1$. We can also estimate total causal effects between observed variables since this causal graph satisfies Assumption \ref{assu:dec}. The output of Algorithm \ref{alg:causalstrengths} would be:

\begin{equation}
\nonumber\begin{bmatrix}
1 & 0.892 &-0.049  \\
-0.042 & 1 & -0.024 \\
-0.032 & -0.02 & 1
\end{bmatrix}.
\end{equation}
which is close to the true causal effects.

\begin{figure}[t]
	\centering
	\includegraphics[width=0.6\textwidth]{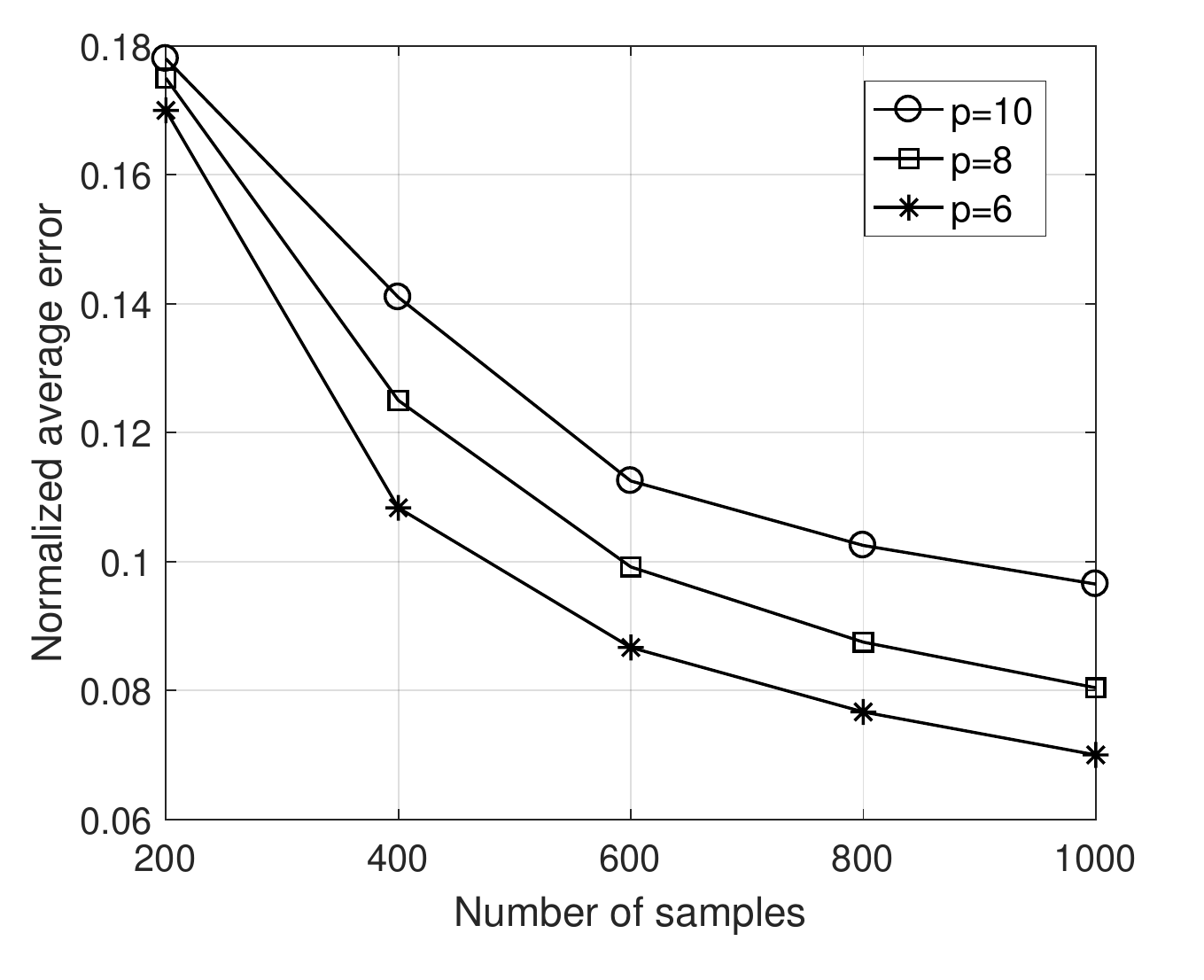}
	\caption{Average normalized error versus number of samples.}\label{fig:4}
\end{figure}

We generated 1000 DAGs of size $p=6,8,10$ by first selecting a causal order among variables randomly and then connecting each pair of variables with probability $0.3$. We generated data from a linear SEM where nonzero entries of matrix $\mathbf{A}$ is equal to $0.9$, and the exogenous noises have a uniform distribution. In each generated DAG, we selected $p_l=p/2$ variables randomly as latent variables. We checked whether there is a causal path between any two observed variables by a similar procedure described for the previous examples. We define normalized error as the number of pairs such as $(V_i,V_j)$ that there exists a causal path from $V_i$ to $V_j$ in the true causal graph but we output that there is no causal path between them (or vice versa) to the total pairs, i.e., $p(p-1)$. In Figure \ref{fig:4}, the average normalized error of the result given by our approach is depicted versus the number of samples. As can be seen, the average normalized error is fairly low for large enough samples. Furthermore,  we have better performance for the cases with smaller number of variables in the system.

\subsection{Real data}
We considered the daily closing prices of the following world stock indicies from $10/12/2012$ to $10/12/2018$, obtained from Yahoo financial database: Dow Jones Industrial Average (DJI) in USA, Nikkei 225 (N225) in Japan, Euronext 100 (N100) in Europe, 
Hang Seng Index (HSI) in Hong Kong, and the Shanghai Stock Exchange Composite Index (SSEC)
in China.

\begin{figure}
	\centering
	\includegraphics[width=0.35\textwidth]{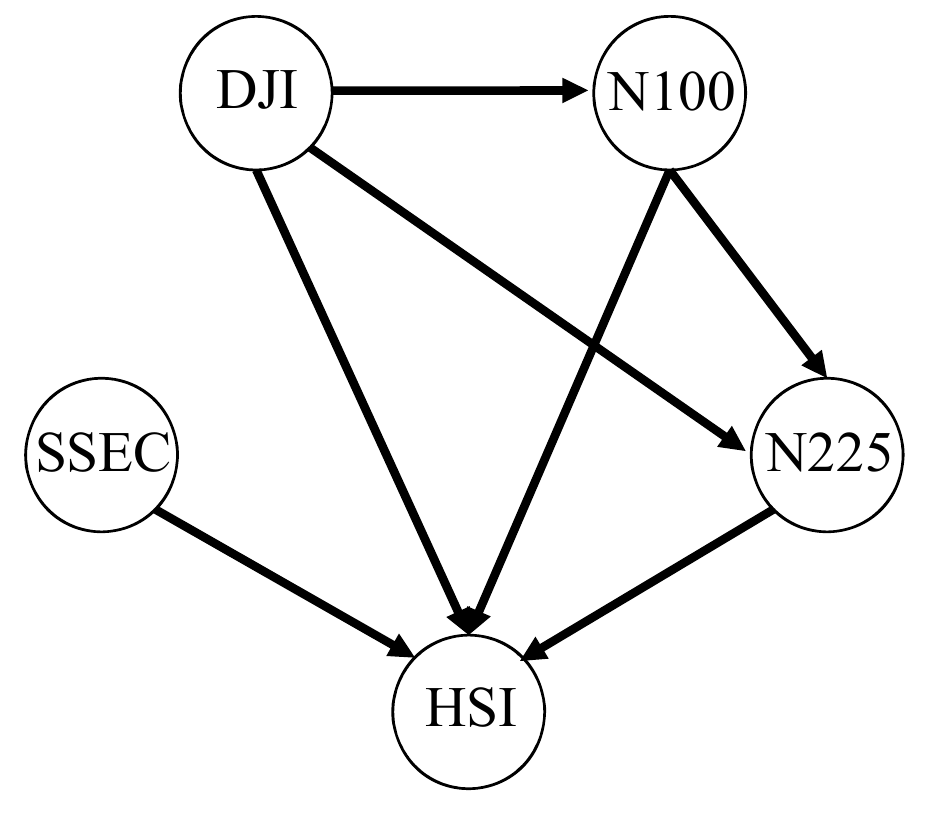}
	\caption{The causal relationships among five world stock indicies obtained from the proposed method in Section \ref{section:path}.}\label{fig:stock}
\end{figure}

Let $c_i(t)$ be the closing price of $i$-th index on day $t$. We define the corresponding return by $R_i(t):=(c_i(t)-c_{i-1}(t))/c_{i-1}(t)$. We considered the returns of indicies as an observational data and applied the proposed method in Section \ref{section:path} in order to check the existence of a causal path between any two indicies. Figure \ref{fig:stock} depicts the causal relationships among the indicies. In this figure, there is a directed edge from index $i$ to index $j$ if we find a causal path from $i$ to $j$. As can be seen, there are causal paths from DJI to HSI, N225, and N100 which is commonly known to be true in the stock market \citep{hyvarinen2010estimation}. Furthermore, HSI is influenced by all other indicies and SSEC only affects HSI which these findings are compatible with the previous results in \citep{hyvarinen2010estimation}. 

\section{Conclusions}
We considered the problem of learning causal models from observational data in linear non-Gaussian acyclic models with latent variables. 
Under the faithfulness assumption, we proposed a method to check whether there exists a causal path between any two observed variables. Moreover, we gave necessary and sufficient graphical conditions to uniquely identify the number of variables in the system. From the information about existence of a directed path, we  obtain a causal order among the observed variables. Next, we considered the problem of estimating total causal effects.
We showed by an example that causal effects among observed variables cannot be identified uniquely even under the faithfulness assumption and non-Gaussianity of exogenous noises. However, we can identify all possible set of total causal effects that are compatible with the observational data efficiently in time. Furthermore, we presented structural conditions under which we can learn total causal effects among observed variables uniquely. Experiments on synthetic data and real-world data showed the effectiveness of our proposed algorithms on learning causal models. 







\appendix

\section*{Appendix A. Proof of Theorem \ref{lemma:absorb}}
``if" part:\\
We say a directed path is latent if all the variables on the path except the endpoint are latent. The ``if" parts of conditions in Theorem \ref{lemma:absorb} can be rewritten as follows:\\
(a) Latent variable $V_{p_o+j}$, $1\leq j\leq p_l$, is absorbable in $\emptyset$ if it has no observable descendant.\\
(b1) Latent variable $V_{p_o+j}$, $1\leq j\leq p_l$, is absorbable in observed variable $V_i$, $1\leq i\leq p_o$, if $V_i$ is the only observed variable influenced by $V_{p_o+j}$ through some latent paths.\\
(b2) Latent variable $V_{p_o+j}$, $1\leq j\leq p_l$, is absorbable in latent variable $V_{p_o+k}$, $1\leq k\leq p_l$, if all latent paths from $V_{p_o+j}$ to observed variables go through $V_{p_o+k}$.
\\ It is easy to show that conditions (b1) and (b2) are equivalent to ``if" part of condition (b) in Theorem \ref{lemma:absorb}.
From \eqref{eq:mainp}, we know that $\mathbf{V_o}= (\mathbf{I}-\mathbf{D})^{-1}[\mathbf{I}|\mathbf{A_{ol}}(\mathbf{I}-\mathbf{A_{ll}})^{-1}]\mathbf{N}$ where entry $(i,j)$ of matrix $ (\mathbf{I}-\mathbf{D})^{-1}\mathbf{A_{ol}}(\mathbf{I}-\mathbf{A_{ll}})^{-1}$ is the total causal effect of latent variable $V_{p_o+j}$ to the observed variable $V_i$. This entry would be zero if no directed path exists from latent variable $V_{p_o+j}$ to observed variable $V_i$. Now, we prove the correctness of above conditions:\\
(a) If a latent variable $V_{p_o+j}$ has no observable descendant, then the $j$-th column of  $\mathbf{A_{ol}}(\mathbf{I}-\mathbf{A_{ll}})^{-1}$ is all zeros. Hence, there would be no changes in $[\mathbf{I}|\mathbf{A_{ol}}(\mathbf{I}-\mathbf{A_{ll}})^{-1}]\mathbf{N}$ by setting $N_{p_o+j}$ to zero. Therefore, there would be no change in $P_{V_o}$.\\
(b1) Since latent variable $V_{p_o+j}$ only influences one observed variable through latent paths, $[\mathbf{A_{ol}}(\mathbf{I}-\mathbf{A_{ll}})^{-1}]_{:,j}$ has only one non-zero entry and therefore linearly dependent on one of columns of identity matrix, let say $i$-th column. Moreover, the total causal effect from $V_{p_o+j}$ to $V_i$, i.e., $[\mathbf{B}]_{i,p_o+j}$ is equal to $[\mathbf{A_{ol}}(\mathbf{I}-\mathbf{A_{ll}})^{-1}]_{i,j}$ since there is no causal path from $V_{p_o+j}$ to $V_i$ that goes through an observed variable other than $V_i$. Thus, we replace $N_i$ by $N_i+[\mathbf{A_{ol}}(\mathbf{I}-\mathbf{A_{ll}})^{-1}]_{i,j} N_{p_o+j}$ and set $N_{p_o+j}$ to zero and there would be no change in $[\mathbf{I}|\mathbf{A_{ol}}(\mathbf{I}-\mathbf{A_{ll}})^{-1}]\mathbf{N}$.\\
(b2)  Consider any observed variable $V_i$, $1\leq i\leq p_o$. If all latent paths of $V_{p_o+j}$ go though $V_{p_o+k}$, then $[\mathbf{A_{ol}}(\mathbf{I}-\mathbf{A_{ll}})^{-1}]_{i,j}=[\mathbf{A_{ol}}(\mathbf{I}-\mathbf{A_{ll}})^{-1}]_{i,k}[\mathbf{B}]_{p_o+k,p_o+j}$ since all the paths from $V_{p_o+j}$ to $V_{p_o+k}$ are latent. Thus, we can change $N_{p_o+k}$ to $N_{p_o+k}+ [\mathbf{B}]_{p_o+k,p_o+j}N_{p_o+j}$ and set $N_{p_o+j}$ to zero and there would be no change in $[\mathbf{I}|\mathbf{A_{ol}}(\mathbf{I}-\mathbf{A_{ll}})^{-1}]\mathbf{N}$.
\\
``only if'' part:\\
Now, we probe that the conditions (a), (b1), and (b2) are the only absorbable case. It can be easily shown that an observed variable cannot be absorbed into any other observed or latent variables. Thus, it is just needed to consider the following cases:
\begin{itemize}
	\item Absorbing a latent variable in an observed variable: Suppose that a latent variable $V_{j}$ can be absorbed in an observed variable $V_{i}$. Furthermore, assume that $V_{j}$ also influences other observed variable $V_k$ through latent path(s). That is, there exist some paths that start from $V_{j}$ and end in $V_{k}$ without traversing, $V_i$. Let $\gamma\neq 0$ be the causal strength of such paths. Then, $[\mathbf{B}]_{k,j}=[\mathbf{B}]_{k,i}\times [\mathbf{B}]_{i,j} +\gamma$. To absorb $V_{j}$ in $V_i$, $\gamma$ should be zero which would contradict  the faithfulness assumption.
	\item Absorbing a latent variable in another latent variable: Suppose that a latent variable $V_{j}$ can be absorbed in another latent variable $V_i$ but for some observed variable $V_k$, all latent paths from $V_j$ do not go through $V_i$. Let $\gamma$ be the causal strength of such paths. Then, $[\mathbf{B}]_{k,j}= [\mathbf{B}]_{k,i}\times [\mathbf{B}]_{i,j}+\gamma$. To absorb $V_{j}$ in $V_i$, $\gamma$ should be zero which contradicts the faithfulness assumption.
\end{itemize}

\section*{Appendix B. Proof of Lemma \ref{lemma:minimal}}
\label{Appendix:B}
Suppose that a latent variable $V_i$ has at least two non-absorbable children such as $V_j$ and $V_k$. We need to consider three cases:  

\begin{itemize}
	\item If both of $V_j$ and $V_k$ are observed variables, then $V_i$ is not absorbable according to Theorem \ref{lemma:absorb}.  
	\item Suppose that $V_j$ and $V_k$ are latent variables. Each of them must reach at least two observed variables through latent paths (due to condition (b) in Theorem \ref{lemma:absorb}). Thus, $V_i$ also reaches those observed variables through latent paths. Furthermore, all of latent paths starting from $V_i$ does not go through only one latent variables. Hence, none of the conditions in Theorem \ref{lemma:absorb} are not satisfied and $V_i$ is not absorbable.
	\item
	One of $V_j$ or $V_k$, let say variable $V_j$, is observed. $V_k$ must reach an observed variable other than $V_j$ through some latent paths. Otherwise, it is absorbable. Therefore, $V_i$ is not absorbable since it does not satisfy any conditions in Theorem \ref{lemma:absorb}. 
\end{itemize}

\section*{Appendix C. Proof of Theorem \ref{cond:eq}}
 If $G$ is not minimal, then it can be easily seen that $\mathbf{B'}$ is also reducible. Now, suppose that $G$ is minimal. We want to show that $\mathbf{B'}$ is also not reducible almost surely. By contradiction, suppose that $\mathbf{B'}$ is reducible. Then two columns of $[\mathbf{I}|\mathbf{A_{ol}}(\mathbf{I}-\mathbf{A_{ll}})^{-1}]$ must be linearly dependent. Now, two cases should be considered: 
 \begin{itemize}
 	\item   One column of $\mathbf{A_{ol}}(\mathbf{I}-\mathbf{A_{ll}})^{-1}$, let say $i$-th column, and one column of $\mathbf{I}$ are linearly dependent. Hence, all the latent paths starting from latent variable $V_{p_o+i}$ influences only one observed variable (Condition (b) in Theorem \ref{lemma:absorb}). Thus, $G$ is not minimal which is a contradiction.
 	\item Two columns of $\mathbf{A_{ol}}(\mathbf{I}-\mathbf{A_{ll}})^{-1}$, let say $i,j$ are linearly dependent. If the corresponding columns have only one non-zero entry, then both of them can be absorbed in an observed variable (Condition (b) in Theorem \ref{lemma:absorb}). Thus, $G$ is not minimal. Now, suppose that these columns have more than one nonzero entry each, let say entries $k$ and $l$. Without loss of generality, suppose that $V_{p_o+i}$ is the ancestor of $V_{p_o+j}$. Let $h_i$ be the maximum length of latent paths starting from latent variable $V_{p_o+i}$. By induction on $h_i$, we will show that $i,j$-th columns of $\mathbf{A_{ol}}(\mathbf{I}-\mathbf{A_{ll}})^{-1}$ are linearly dependent with measure zero. The case of $h_i=1$ is trivial. Suppose that for $h_i=r$, the statement holds true. We will prove it for $h_i=r+1$. Let latent variable $V_{p_o+u}$ be a child of $V_{p_o+i}$ and assume some paths from $V_{p_o+u}$ do not go through $V_{p_o+j}$. We know that: 
 	\begin{equation}
 	[\mathbf{B}]_{k,p_o+j}/[\mathbf{B}]_{l,p_o+j}=[\mathbf{B}]_{k,p_o+i}/[\mathbf{B}]_{l,p_o+i}.
 	\label{eq:p1}
 	\end{equation} 
 	Furthermore, 
 	\begin{equation}
 	[\mathbf{B}]_{k,p_o+i}=[\mathbf{B}]_{k,p_o+u}[\mathbf{B}]_{p_o+u,p_o+i}+c',[\mathbf{B}]_{l,p_o+i}=[\mathbf{B}]_{l,p_o+u}[\mathbf{B}]_{p_o+u,p_o+i}+c'',
 	\label{eq:p2}
 	\end{equation}
 	for some values $c',c''$. Moreover, $[\mathbf{B}]_{p_o+u,p_o+i}= [\mathbf{A}]_{p_o+u,p_o+i}+c'''$ for some $c'''$. Plugging \eqref{eq:p2} into \eqref{eq:p1}, we have:
 	\begin{align*}
 	([\mathbf{B}]_{k,p_o+u}[\mathbf{B}]_{l,p_o+j}&-[\mathbf{B}]_{k,p_o+j}[\mathbf{B}]_{l,p_o+u})[\mathbf{A}]_{p_o+u,p_o+i}=\\ &[\mathbf{B}]_{l,p_o+j}c'-[\mathbf{B}]_{k,p_o+j}c''- ([\mathbf{B}]_{k,p_o+u}[\mathbf{B}]_{l,p_o+j}-[\mathbf{B}]_{k,p_o+j}[\mathbf{B}]_{l,p_o+u})c'''.
 	\end{align*}
 	The above equation holds with measure zero if $[\mathbf{B}]_{k,p_o+u}[\mathbf{B}]_{l,p_o+j}-[\mathbf{B}]_{k,p_o+j}[\mathbf{B}]_{l,p_o+u}\neq0$ which is true with measure one from the induction hypothesis.
 \end{itemize}

\section*{Appendix D. An Example of non-Identifiability of Total Causal Effects}
\label{Appendix:A}
Let $P=(V_{i_0}, V_{i_1}, \cdots,  V_{i_{r-1}}, V_{i_r})$ be a causal path of length $r$ from variable $V_{i_0}$ to variable $V_{i_r}$. We define the weight of path $P$, denoted by $\omega_p$ , as the product of direct causal strengths of edges on the path:
\begin{equation}
\omega_P=\prod_{s=0}^{r-1} [\mathbf{A}]_{i_{s+1},i_{s}}.
\end{equation}

Suppose that $\Pi_{V_i,V_j}$ be the set of all causal paths from variable $V_i$ to variable $V_j$. It can be shown that the total causal effect from $V_i$ to $V_j$ can be computed by the following equation:
\begin{equation}
[\mathbf{B}]_{j,i} = \sum_{P\in \Pi_{V_i,V_j} }\omega_P.
\end{equation}

Now, consider a causal graph in Figure \ref{fig:example} where $V_i$ and $V_j$ are observed variables and $V_k$ is latent variable. There exist causal paths from $V_k$ to $V_i$ and $V_j$, and from $V_i$ to $V_j$ with the following properties:
\begin{itemize}
	\item Let $\Pi'_{V_k,V_j}$ be the causal paths from variable $V_k$ to variable $V_j$ where $V_i$ is not on any of these paths. We assume that $\Pi'_{V_k,V_j}\neq \emptyset$.
	\item All intermediate variables in $\Pi_{V_k,V_i}$, $\Pi'_{V_k,V_j}$ and $\Pi_{V_i,V_j}$ are latent.
\end{itemize}

We can write $V_i$ and $V_j$ based on the exogenous noises of their ancestors as follows:
\begin{align}
\begin{split}
V_i &= \alpha N_k+ \sum_{V_r\in anc(V_i)\backslash V_k} [\mathbf{B}]_{i,r} N_r,\\
V_j &= \beta N_i +\gamma N_k + \sum_{V_r\in anc(V_j)\backslash \{V_k,V_i\}} [\mathbf{B}]_{i,r} N_r,
\label{eq:example}
\end{split}
\end{align}
where $\alpha=\sum_{P\in \Pi_{V_k,V_i}} \omega_P$, $\beta=\sum_{P\in \Pi_{V_i,V_j}} \omega_P$, and $\gamma=\sum_{P\in \Pi'_{V_k,V_j}} \omega_P$.

Now, we construct a causal graph depicted in Figure \ref{fig:example} where the exogenous noises of variables $V_i$ and $V_k$ are changed to $\alpha N_k$ and $N_i$, respectively. Furthermore, we pick three paths $P_1\in \Pi_{V_k,V_i}$, $P_2\in \Pi'_{V_k,V_j}$, $P_3\in \Pi_{V_i,V_j}$ where:
\begin{align}
\nonumber P_1&=(V_k, V_{u_1},\cdots, V_i),\\
\nonumber P_2&=(V_k, V_{u_2},\cdots, V_j),\\
\nonumber P_3&=(V_i, V_{u_3},\cdots, V_j).
\end{align}

By our first property on the paths, we can find two paths $P_1$ and $P_2$ such that $V_{u_1}\neq V_{u_2}$. We also change the matrix $\mathbf{A}$ to matrix $\mathbf{A}'$ where all the entries of $\mathbf{A}'$ are the same as $\mathbf{A}$ except three entries $[\mathbf{A}']_{u_1,k}$, $[\mathbf{A}']_{u_2,k}$, and 
$[\mathbf{A}']_{u_3,i}$. We will adjust these three entries such that the total causal effects from $V_k$ to $V_i$, from $V_k$ to $V_j$, and from $V_i$ to $V_j$ become 1, $-\gamma/\alpha$, and $\beta +\gamma/\alpha$, respectively. Moreover, these adjustments should not change the dependencies of observed variables $V_i$ and $V_j$ to the exogenous noises of their ancestors given in Equation \eqref{eq:example}. It can be shown that we can change the three mentioned causal effects to our desired values by the following adjustments:
\begin{align*}
[\mathbf{A}']_{u_1,k}=\frac{1-\sum_{P\in \Pi_{V_k,V_i}\backslash\{P_1\}} \omega_{P}}{\omega_{P_1}/[\mathbf{A}]_{u_2,k}},\\
[\mathbf{A}']_{u_2,k}=\frac{-\gamma/\alpha-\sum_{P\in \Pi'_{V_k,V_j}\backslash\{P_2\}} \omega_P}{\omega_{P_2}/[\mathbf{A}]_{u_2,k}},\\
[\mathbf{A}']_{u_3,i}=\frac{\beta+\gamma/\alpha-\sum_{P\in \Pi_{V_i,V_j}\backslash\{P_3\}} \omega_P}{\omega_{P_3}/[\mathbf{A}]_{u_3,i}}.
\end{align*} 

Now, consider any latent variable $V_u$ which is on one of the paths in $\Pi_{V_k,V_i}$, $\Pi'_{V_k,V_j}$, or $\Pi_{V_i,V_j}$. Changes in those mentioned three edges cannot affect the total causal effect from $V_u$ to $V_i$ or $V_j$ since the edges $(V_k,V_{u_1})$, $(V_k,V_{u_2})$, and $(V_i,V_{u_3})$ are not a part of any paths from $V_u$ to $V_i$ or $V_j$. Thus, equations in \eqref{eq:example} do not change while the total causal effect from $V_i$ to $V_j$ becomes $\beta+\gamma/\alpha$ in the second causal graph. It is noteworthy to mention that changes in the equations of latent variables are not important since we are not observing these variables.

\vskip 0.2in
\bibliography{ref}

\end{document}